\documentclass[11pt,english]{article}
\usepackage[T1]{fontenc}
\usepackage[latin9]{inputenc}
\usepackage{float}
\usepackage{amsmath}
\usepackage{amsthm}
\usepackage{amssymb}
\usepackage{graphicx}
\usepackage[authoryear]{natbib}
\usepackage{xcolor}
\DeclareGraphicsExtensions{.pdf}
\makeatletter

\providecommand{\tabularnewline}{\\}
\floatstyle{ruled}
\newfloat{algorithm}{tbp}{loa}
\providecommand{\algorithmname}{Algorithm}
\floatname{algorithm}{\protect\algorithmname}

\theoremstyle{plain}
\newtheorem{thm}{\protect\theoremname}
\theoremstyle{plain}
\newtheorem{lem}[thm]{\protect\lemmaname}
\ifx\proof\undefined
\newenvironment{proof}[1][\protect\proofname]{\par
	\normalfont\topsep6\p@\@plus6\p@\relax
	\trivlist
	\itemindent\parindent
	\item[\hskip\labelsep\scshape #1]\ignorespaces
}{%
	\endtrivlist\@endpefalse
}
\providecommand{\proofname}{Proof}
\fi

\usepackage{jmlr2e_arxiv}

\jmlrheading{1}{2020}{1-48}{4/00}{10/00}{xxxxx}{}

\ShortHeadings{A Momentumized, Adaptive, Dual Averaged Gradient Method}{}
\firstpageno{1}

\author{\name Aaron Defazio \\
       \addr Facebook AI Research, New York \\
       \AND
       \name Samy Jelassi \\
       \addr Princeton University, Princeton}

\editor{Kevin Murphy and Bernhard Sch{\"o}lkopf}

\usepackage{algorithm,algpseudocode}

\makeatother

\usepackage{babel}
\providecommand{\lemmaname}{Lemma}
\providecommand{\theoremname}{Theorem}

\begin{document}
\title{Adaptivity without Compromise: A Momentumized, Adaptive, Dual Averaged
Gradient Method for Stochastic Optimization}
\maketitle
\begin{abstract}
We introduce MADGRAD, a novel optimization method in
the family of AdaGrad adaptive gradient methods. MADGRAD shows excellent
performance on deep learning optimization problems from multiple fields,
including classification and image-to-image tasks in vision, and recurrent
and bidirectionally-masked models in natural language processing.
For each of these tasks, MADGRAD matches or outperforms both SGD and
ADAM in test set performance, even on problems for which adaptive
methods normally perform poorly.
\end{abstract}

\section{Introduction}

Optimization for deep learning forms a relatively new and growing
sub-field in the optimization community. Compared to classical first
order optimization, deep learning problems introduce additional concerns
which require new tools to overcome. Deep learning problems are characterized
by very large parameter vector sizes $D$, making it computationally
infeasible to store matrices of size $D\times D$, and even ``limited
memory'' approaches can be impractical for problems such as the 100+
billion parameter models currently being explored \citep{rajbhandari2019zero,brown2020language}.
The practical limit on these problems is storage that is fixed at
a small multiple of the parameter vector size.

For this reason, diagonal scaling approaches have become the industry
standard for deep learning. In this class of methods, adaptivity is
performed independently for each coordinate, so that memory usage
scales as $O(D)$. We consider Adam \citep{kingma2014adam} the benchmark
method in this class; it has seen widespread adoption, and there are
no alternative adaptive methods that consistently out-perform it \citep{choi2020empirical,schmidt2020descending}.

Adam builds upon a rich history of diagonal adaptive methods. The
AdaGrad method \citep{duchi2011adaptive} introduced a principled
approach to diagonal adaptivity, that arises naturally as a simplification
of a full-matrix adaptivity scheme. This approach is clearly motivated
and yields natural convergence rate bounds for convex losses. Also
within this family, the RMSProp method \citep{rmsprop2012} arose
as a well-performing empirical method in this class, albeit with little
theoretical motivation. The development of the Adam method can be
seen as a natural extension of the scaling used in RMSProp to include
a form of momentum, as well as a stabilizing ``bias-correction''
that significantly dampens the adaptivity and step-size during the early stages
of optimization.

Despite its widespread success, Adam is far from a panacea for deep
learning optimization. \citet{marginal_adaptive2017} show that Adam
as well as other common adaptive optimizers converge to bad local
minima on some important problems, such as the widely studied problem
of image classification. This has led to the general claim that adaptive
methods generalize poorly. As we will show, this is not necessarily
the case. The method we develop in this work combines adaptivity with
strong generalization performance.

Our MADGRAD (Momentumized, Adaptive, Dual averaged GRADient)
method performs consistently at a state-of-the-art level across a
varied set of realistic large-scale deep learning problems, without
requiring any more tuning than Adam. MADGRAD is constructed from the
lesser-used dual averaging form of AdaGrad, through a series of direct
and systematic changes that adapt the method to deep learning optimization.

\section{Problem Setup}

We consider the stochastic optimization framework, where the goal
is to minimize a parameterized function 
\[
f(x)=\mathbb{E}_{\xi}\left[f(x,\xi)\right],
\]
 where $x\in\mathbb{R}^{D}$, and each $\xi$ is a random variable
drawn from a fixed known distribution. In the case of empirical risk
minimization, $\xi$ is a data-point drawn from the data distribution,
typically further processed by a stochastic data-augmentation procedure.
At each step $k$, a stochastic optimization algorithm is given $\xi_{k}$
and has access to $f(x_{k},\xi_{k})$ and $\nabla f(x_{k},\xi_{k})$
for a pre-specified iterate $x_{k}$.

\section{Related Work}

The theory of adaptive methods for non-convex optimization is still
in its infancy. The current best known convergence theory for Adam
due to \citet{defossez2020simple} greatly improves over earlier theory
\citep{zou2019a}, but has the important caveat that it requires momentum
values of the order $\beta=1-1/N$ for $N$ iterations, which is far
from the values used in practice, which are of the order $\beta=0.9$
to $\beta=0.99$. Results for these settings may not be possible,
as \citet{reddibeyondadam2018} show via a counter-example that Adam
may fail to converge under common parameter settings, even in the
convex case. When $\beta_1$ \& $\beta_2$ are small, the Adam update is close to 
sign-sgd (i.e. $x_{k+1}=x_{k}-\gamma\text{sign}(\nabla f(x_{k},\xi_{k})$),
a method that also fails to converge in the general stochastic case
\citep{dissecting2018}, although some theory is possible under a
large batch assumption \citet{pmlr-v80-bernstein18a} where the behavior
is closer to the non-stochastic case.

AdaGrad's convergence in the non-convex case has also been studied.
\citet{pmlr-v97-ward19a} establish convergence for a restricted variant
where only a global step size is adaptively updated. \citet{pmlr-v89-li19c} establish almost sure convergence for a variant of AdaGrad where the most recently seen gradient is omitted from the denominator. Convergence with high probability is also established for a variant with global rather than coordinate-wise step size. More recently
\citet{zhou2020adaptive} and \citet{zou2019weighted} establish convergence
of non-momentum and momentum variants respectively, although with
bounds that are much worse than established by \citet{defossez2020simple},
who also cover AdaGrad in their analysis.

Weighted AdaGrad as we use in this work has been explored to varying degrees before, including
the non-convex case in the aforementioned work by \citet{zou2019weighted},
and the convex case by \citet{accel_adagrad2018}. Weighting is particularly
interesting in the strongly convex case, where weights such as $\lambda_{k}\propto k^{2}$
can be used to achieve accelerated convergence. Neither of these works
cover the dual averaged form of AdaGrad which we explore.


\section{Adaptivity in deep learning beyond Adam}
To understand the motivation and design of the MADGRAD method, a clear understanding of the short-comings of existing methods is needed. Consider Adam, the most heavily used adaptive method in practice. Although it works remarkably well on some important problems, it also suffers from the following issues:
\begin{itemize}
    \item It greatly under-performs the non-adaptive SGD-M method in a number of important situations including the widely studied ImageNet training problem.
    \item Problems can be constructed on which it will fail to converge entirely, even in the convex setting.
    \item The exponential moving average updates used are non-sparse when given sparse gradients, which makes the method poorly suited to sparse problems.  
\end{itemize}
Due to these issues, Adam doesn't quite reach the goal of being a general-purpose deep learning optimizer. The MADGRAD method is directly designed to address these issues. MADGRAD:
\begin{itemize}
    \item Achieves state-of-the-art performance across problems traditionally tackled by Adam, while simultaneously achieving state-of-the-art on problems where Adam normally under-performs.
    \item Has provable and strong convergence theory on convex problems.
    \item Is directly applicable to sparse problems when momentum is not used.
\end{itemize}

\section{Design}


The MADGRAD method is the combination of a number of techniques that
individually address separate short-comings in the AdaGrad method
when applied to deep learning optimization problems. By building upon a method with known convergence theory, we are able to construct a method that is still provably convergent (under convexity assumptions) without sacrificing the practical performance characteristics of Adam. We will detail each of these techniques in turn, to build up MADGRAD from its foundations.

\subsection{Dual averaging for deep learning}


MADGRAD is based upon the dual averaging formulation of AdaGrad, rather
than the mirror descent formulation. Although the original seminal
work on AdaGrad \citep{duchi2011adaptive} presents the dual averaging
formulation with equal weight as the mirror descent form, the dual
averaging form has seen virtually no use for deep learning optimization.
The AdaGrad implementations available in major deep learning frameworks
(PyTorch, Tensorflow) contain the mirror descent form only. This is
despite the theory presented for the dual averaging formulation being
arguably more elegant than the mirror descent theory. The dual averaging form of AdaGrad satisfies the following bound:
\[
\sum_{i=1}^{k}f(x_{i})-f(x_{*})\leq\frac{1}{\gamma}\psi_{k}(x_{*})+\frac{\gamma}{2}\sum_{i=1}^{k}\left\Vert \nabla f_{i}(x_{i})\right\Vert _{\psi^{*}_{i-1}}^{2}
\]
Whereas the mirror descent form satisfies the following more complex bound, involving the Bregman divergence of $\psi$:
\begin{gather*}
\sum_{i=1}^{k}f(x_{i})-f(x_{*})\leq\\
\frac{1}{\gamma}B_{\psi_{1}}(x_{*},x_{1})+\frac{1}{\gamma}\sum_{i=1}^{k-1}\left[B_{\psi_{i+1}}(x_{*},x_{i+1})-B_{\psi_{i}}(x_{*},x_{i+1})\right]+\frac{\gamma}{2}\sum_{i=1}^{k}\left\Vert \nabla f_{i}(x_{i})\right\Vert _{\psi_{i}^{*}}^{2}.
\end{gather*}
Given the clear advantage in terms of theoretical simplicity, why then are dual averaging approaches not used more widely? We believe this is due to a number of misconceptions. The first misconception
is that dual averaging is only interesting in the composite optimization
setting, where sophisticated regularizers are used to encourage sparsity
or induce other properties of the solution. It is true that
for smooth non-stochastic optimization, gradient descent and
mirror descent coincide (under optimal hyper-parameters). However, when the
objective is stochastic or non-smooth, the methods become distinct, and actually behave quite differently.

Dual averaging has the general form, given a proximal function $\psi$:
\begin{align}
g_{k} & =\nabla f\left(x_{k},\xi_{k}\right),\nonumber \\
s_{k+1} & =s_{k}+\lambda_{k}g_{k},\nonumber \\
x_{k+1} & =\arg\min_{x}\left\{ \left\langle s_{k+1},x\right\rangle +\beta_{k+1}\psi(x)\right\} .\label{eq:vanilla_da}
\end{align}
The gradient buffer $s_0$ is initialized as the zero vector. The simplest form of dual averaging occurs when the standard Euclidean
squared norm is used: $\psi(x)=\frac{1}{2}\left\Vert x-x_{0}\right\Vert ^{2}$, and $\lambda_{k}=1$
in which case the method takes the form:
\begin{equation}
x_{k+1}=x_{0}-\frac{1}{\beta_{k+1}}\sum_{i=0}^{k}g_{i}\label{eq:da_simple}.
\end{equation}
If the objective is either non-smooth or stochastic (or
both), $\beta$ sequences of the form $\beta_{k+1}=\sqrt{k+1}$ give
a convergent method. Although Equation \ref{eq:da_simple} has little
resemblance to SGD as written, SGD's update:
\[
x_{k+1}=x_{k}-\gamma_{k}\nabla f\left(x_{k},\xi_{k}\right),
\]
can be written in the more comparable form:
\begin{equation}
x_{k+1}=x_{0}-\sum_{i=0}^{k}\gamma_{i}g_{i}.\label{eq:sgd_sum_form}
\end{equation}
where to achieve convergence without a fixed stopping time, a step
size of the form $\gamma_{i}\propto1/\sqrt{i+1}$ is standard. Comparing
SGD and DA at a step $k$, it's clear that the weighting sequence
used by SGD places a smaller weight on newer $g_i$ in the summation compared to earlier $g_i$,
whereas the sequence used by DA places equal weight on all $g_i$.
This difference is key to understanding why methods in the DA family behaves differently
from SGD in practice, even without additional regularization or non-Euclidean
proximal functions.

The second misconception arises from implementing the dual averaging
form of AdaGrad without considering what modifications need to be
made for the deep learning setting. The algorithm as originally stated,
uses an initial point of the origin $x_{0}=0$, and a proximity
function $\psi_{t}(x)=\frac{1}{2}\left\langle x,H_{t}x\right\rangle $
that is quadratic, but centered around the origin. It is well known
that neural network training exhibits pathological behavior when initialized
at the origin, and so naive use of this algorithm does not perform
well. When centering around 0, we have observed severely degraded empirical performance and a high risk of divergence. Instead, a proximity function centered about $x_{0}$ needs
to be used:
\[
\psi_{t}(x)=\frac{1}{2}\left\langle x-x_{0},H_{t}\left(x-x_{0}\right)\right\rangle,
\]
with initialization of $x_{0}$ following standard conventions for the
network being trained.

\subsection{Dual averaging generalizes well}
In addition to the theoretical advantages of dual averaging methods, we have also observed that they also enjoy a strong practical advantage in the form of better generalization performance.
Dual averaging based methods include a form of implicit regularization,
which we believe is a crucial factor contributing to their good generalization
performance. To see this, consider the classical
dual averaging update:
\[
x_{k+1}=x_{0}-\frac{1}{\sqrt{k+1}}\sum_{i=0}^{k}g_{i},
\]
This update can be written in a form closer to the SGD update by substituting
for $x_{0}$:
\begin{align*}
x_{k+1} & =\left(x_{k}+\frac{1}{\sqrt{k}}\sum_{i=0}^{k-1}g_{i}\right)-\frac{1}{\sqrt{k+1}}\sum_{i=0}^{k}g_{i},\\
 & =x_{k}-\frac{1}{\sqrt{k+1}}\left[g_{k}-\left(\frac{\sqrt{k+1}}{\sqrt{k}}-1\right)\sum_{i=0}^{k-1}g_{i}\right],\\
 & = x_{k}-\frac{1}{\sqrt{k+1}}\left[g_{k}+\left(\sqrt{k+1}-\sqrt{k}\right)\left(x_{k}-x_{0}\right)\right].
\end{align*}
Since $\sqrt{k+1}-\sqrt{k}\approx1/(2\sqrt{k+1})$, the behavior of dual
averaging resembles a SGD step with a step-dependent regularizer:
\[
\frac{1}{4\sqrt{k}}\left\Vert x_{k}-x_{0}\right\Vert ^{2},
\]
which decays in strength during the course of optimization. We speculate that the indirect
decaying regularization inherent in dual averaging methods may explain
why MADGRAD also requires less decay than other methods to match their
performance. The strong initial regularization may have a positive effect during early iterations, while not negatively affecting the ability of the model to fit to the data during the later "fine-tuning" epochs. Given the practical advantages we observe in our experiments, we believe further research into the effect of using stronger regularization at the early stages of optimization may be interesting more generally.

\subsection{$\lambda$ sequences for deep learning}


\label{subsec:lamb-seq}
\begin{figure}
\includegraphics[width=0.49\textwidth]{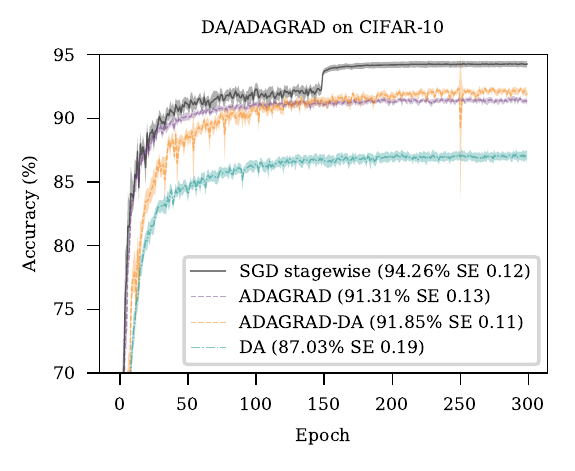}\includegraphics[width=0.49\textwidth]{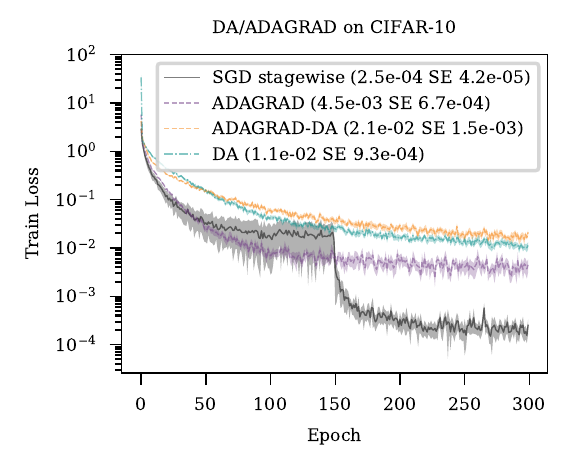}

\includegraphics[width=0.49\textwidth]{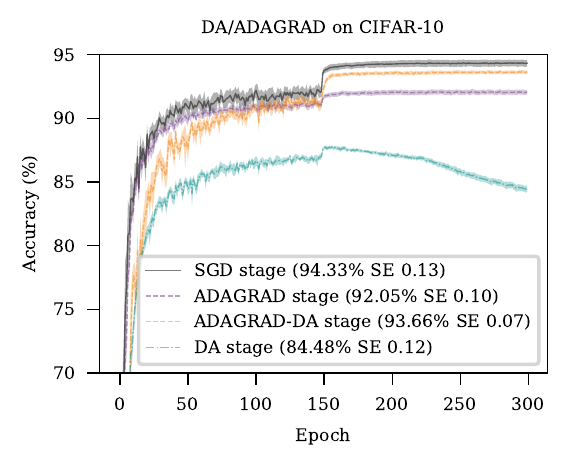}\includegraphics[width=0.49\textwidth]{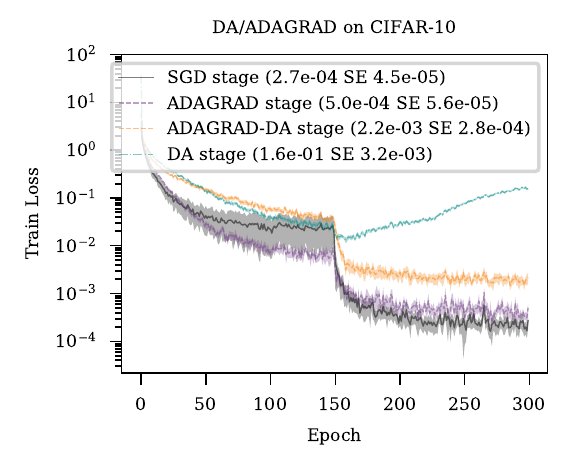}

\caption{\label{fig:da-adagrad}Comparison of SGD without momentum to DA and
DA-AdaGrad and AdaGrad on CIFAR-10. Left column is test classification performance, right column is training loss. The "stage" learning rate scheme involves a 10 fold decrease in the learning rate at epochs 150 and 225. See Section \ref{sec:experiments}
for a full description of the experimental setup.}
\end{figure}
Even with this modification, dual averaging both with and without
adaptivity is not competitive with SGD on standard benchmark problems
such as CIFAR10, as shown in Figure \ref{fig:da-adagrad}.  The top
row shows AdaGrad and DA methods using a flat learning rate schedule,
and the bottom row shows a stage-wise schedule. SGD is shown as a
baseline. For the DA family methods, $\lambda_k$ is decreased for the
stage-wise schedules. Both AdaGrad, DA and AdaGrad-DA under-perform
SGD with either learning rate schedule. Part of this performance gap can be attributed to the fact that each of these methods either implicitly or explicitly use a $1/\sqrt{i+1}$ learning rate sequence.
This sequence is actually \textbf{harmful}, as we can confirm
by testing SGD using a schedule of the form: 
\[
\gamma_{i}=\frac{a}{\sqrt{i+b}},
\]
 Figure \ref{fig:lrsched} illustrates the learning curves achievable
for varying $b$ values on CIFAR-10. Full description of our experimental
setup is in Section \ref{sec:experiments}. We performed a hyper-parameter
search over $a$ separately for each $b$, with test accuracy as the target quantity. All sqrt-decay sequences
are significantly worse than the baseline stage-wise schedule, where
the learning rate is decreased 10 fold at epochs 150 and 225. We speculate that the
sqrt-decay sequences result in convergence that is too rapid, skipping
over the initial annealing stage of learning, resulting in convergence
to a poor local minima.
\begin{figure}
\includegraphics[width=0.49\textwidth]{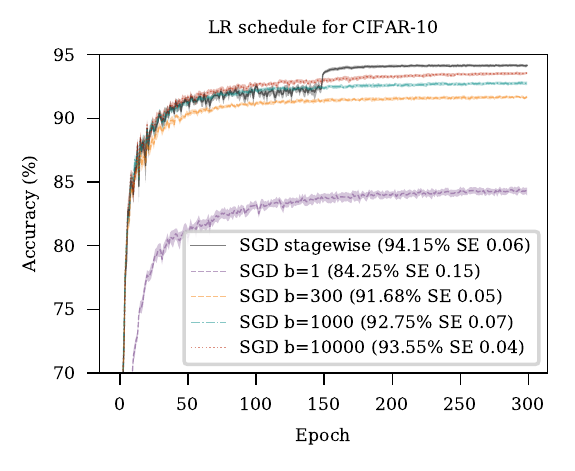}\includegraphics[width=0.49\textwidth]{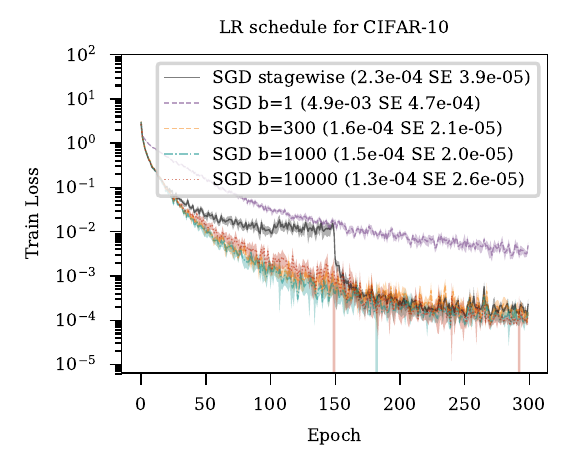}

\caption{\label{fig:lrsched}Sqrt-decay learning rate schedules under-perform
stage-wise schedules. With batch-size 128 on CIFAR-10. No momentum
is used in this comparison. A range of offsets $b$ in the rate $a/\sqrt{i+b}$
were tried with values up to 10,000 shown. Larger values of $b$ up
to 100,000 were also tested, they also failed to match the performance
of the stage-wise schedule. Left column is test classification performance, right column is training loss.}
\end{figure}
The AdaGrad and AdaGrad-DA methods also use an implicitly decreasing
sequence, although the rate of decrease depends on the magnitude of
the gradients, which is very problem dependent. If gradients stay
of similar magnitude over a particular time-scale, then the rate of
decrease will also be a $1/\sqrt{k}$ rate for step $k$. This step size scheme is also undesirable as prevents the
use of standard SGD \& Adam step size sequences for choosing the explicit
step size constants $\lambda_{i}$. 

Since in practice the same learning rate scheme is commonly used when
comparing different optimization methods, this schedule contributes to
the commonly held perception that AdaGrad is not as effective as other
adaptive methods such as Adam. 

For the DA method, we propose to remedy this issue by introducing
a scaling of the $\lambda$ values to counter-act the step size sequence.
In particular we propose the choice:
\[
\lambda_{i}=\left(i+1\right)^{1/2}\gamma_{i},
\]
where $\gamma$ is a conventional (SGD/Adam) step size sequence. The
advantage of this choice is that the leading term in the sum in Equation
\ref{eq:da_simple} has constant weight across $k$:
\begin{align*}
x_{k+1} & =x_{0}-\frac{1}{\sqrt{k+1}}\sum_{i=0}^{k}\lambda_{i}g_{i},\\
 & =x_{0}-\gamma_{k}g_{k}-\frac{1}{\sqrt{k+1}}\sum_{i=0}^{k-1}\lambda_{i}g_{i},
\end{align*}
mirroring the behavior of SGD during a constant step size phase, but
retaining the $\sqrt{k+1}$ decay of past gradients. This simple change
is sufficient to greatly improve the test-set performance of DA when
using the same learning rate schedule as SGD.

Another advantage of this sequence is that it will place higher weights
on latter gradients in the final convergence rate bound. This makes
no difference if we expect gradients to be of similar magnitude at
all stages of optimization (which can happen for non-smooth problems
in the worse case), but in practice even for non-smooth objectives
the gradient typically shrinks to some degree during optimization,
leading to tighter bounds when using a forward weighted lambda sequence.
We discuss this difference further in Section \ref{sec:theory}.

\subsection{Momentum}


The use of momentum on top of SGD is known to be highly beneficial, if not crucial,
for deep learning optimization across a wide variety of architectures
and problem settings \citep{sutskever2013importance}. Given how crucial
it can be to maintaining competitive performance, we now examine how
we can add a form of momentum to the dual averaging updates, and latter
the AdaGrad updates.


We will consider an update of the following form, which was first
explored in this general form by \citet{nesterov2015quasi} under
the name Dual Averaging with Double Averaging:
\begin{align}
g_{k} & =\nabla f\left(x_{k},\xi_{k}\right),\nonumber \\
s_{k+1} & =s_{k}+\lambda_{k}g_{k},\nonumber \\
z_{k+1} & =\arg\min_{x}\left\{ \left\langle s_{k+1},x\right\rangle +\beta_{k+1}\psi(x)\right\} ,\label{eq:da_mom}\\
x_{k+1} & =\left(1-c_{k+1}\right)x_{k}+c_{k+1}z_{k+1}.\nonumber 
\end{align}
The essential idea behind this algorithm is simple. Instead of evaluating
the gradient at each step at the value of the argmin operation as
with regular DA, instead it's evaluated at a moving average point
instead. This serves to smooth the iterate sequence. This technique
has the advantage in the convex setting of making it possible to prove
convergence properties of the last iterate $x_{k+1}$ rather than
the average iterate $\bar{x}_{k+1}=\frac{1}{k+1}\sum_{i=0}^{k}x_{i}$.
Essentially the averaging operation is incorporated into the algorithm
itself. 

Momentum is normally thought of as performing more than just a smoothing
of the iterate sequence, although a line of recent research has shown
that inline averaging of the above form is actually exactly equivalent
to momentum \citep{sebbouh2020convergence,defazio2020understanding}.
This is clearly illustrated when momentum is added on top of SGD,
where inline averaging:
\begin{align*}
z_{k+1} & =z_{k}-\eta_{k}\nabla f(x_{k},\xi_{k}),\\
x_{k+1} & =\left(1-c_{k+1}\right)x_{k}+c_{k+1}z_{k+1},
\end{align*}
is actually exactly equivalent to more common equational forms of
momentum:
\begin{align*}
m_{k+1} & =\beta_{k}m_{k}+\nabla f(x_{k},\xi_{k}),\\
x_{k+1} & =x_{k}-\alpha_{k}m_{k+1},
\end{align*}
for appropriate choices of the hyper-parameters. In the convex setting
the advantage of this form arises when $c_{k+1}=\frac{1}{k+1}$, which
corresponds to an equal weighted moving average $x_{k+1}=\frac{1}{k+1}\sum_{i=0}^{k}z_{i}$.
Under this setting convergence of the last iterate can be shown just
as when this kind of averaging is used with dual averaging \citep{defazio2020power}.
In the non-convex setting, constant $c_{k+1}$ values, which correspond
to an exponential moving average, appear to be the best choice \citep{defazio2020understanding}.

\subsection{Adaptivity}
Our goal is to combine these ideas together with the adaptivity technique
from the AdaGrad method. The dual averaging form of coordinate-wise
AdaGrad has the following form:
\[
x_{k+1}=x_{0}-\frac{1}{\sqrt{\sum_{i=0}^{k}\gamma_{i}g_{i}^{2}}}\circ\sum_{i=0}^{k}\gamma_{i}g_{i},
\]
where $\circ$ represents the element-wise (Hadamard) product, and
$\gamma$ is a fixed step size hyper-parameter. There are many different
ways of combining this kind of coordinate-wise adaptivity with the weighted
gradient sequence $\lambda_{i}=\sqrt{i+1}$ that we have proposed.
Due to the flexibility of the dual averaging framework, it's possible
to prove a convergence rate of some form for practically any choice
of denominator sequence. However, we must take into consideration
that we also want to maintain the magnitude of the ``effective''
step size, as discussed in Section \ref{subsec:lamb-seq}. 

We also need to ensure that the weighted dominator includes $\gamma_{i}$
not just $\sqrt{i+1}$, as this mitigates a problem illustrated for
DA in Figure \ref{fig:da-adagrad}: when $\lambda$ is decreased 10
fold at epoch 150, the method starts to diverge. At this point, the
$\beta$ sequence continues to decrease at a square-root rate, while
the sum-of-gradients starts growing ten times slower. This results
in the method shrinking the iterates towards $x_{0}$ far to strongly.

We review a number of possible alternatives below and discuss their practicality.

\subsubsection{Unweighted denominator }
One possibility is keep the denominator the same but just weight the
gradients in the sum:
\[
x_{k+1}=x_{0}-\frac{1}{\sqrt{\sum_{i=0}^{k}\gamma_{i}g_{i}^{2}}}\circ\sum_{i=0}^{k}\left(i+1\right)^{1/2}\gamma_{i}g_{i},
\]

This is appealing as it maintains the constant effective step size
property, however the resulting convergence rate bound derivable from
this form depends on $\sqrt{\sum_{i=0}^{k}\gamma_{i}g_{i}^{2}}$ rather
than $\sqrt{\sum_{i=0}^{k}\left(i+1\right)^{1/2}\gamma_{i}g_{i}^{2}}$,
which defeats the purpose of using a front-weighted gradient sequence.

\subsubsection{Weighted denominator}
We can weight the gradient sequence in the denominator by $\lambda$
also:
\[
x_{k+1}=x_{0}-\frac{1}{\sqrt{\sum_{i=0}^{k}\left(i+1\right)^{1/2}\gamma_{i}g_{i}^{2}}}\circ\sum_{i=0}^{k}\left(i+1\right)^{1/2}\gamma_{i}g_{i}.
\]
This form does not maintain a constant effective step size, which
results in poor empirical performance. We experimented with mitigations
such as adding additional terms to the numerator that would counteract
this growth, however this still resulted in unsatisfactory empirical
results. 

\subsubsection{Weighted numerator}
The AdaGrad variant proposed by \citet{zou2019weighted} uses a weighting
scheme where the weights $\lambda_{k}$ are included in the numerator
as well as the denominator:
\[
x_{k+1}=x_{0}-\frac{\gamma_{i}}{\sqrt{t}}\frac{\sqrt{\sum_{i=0}^{k}\lambda_{i}}}{\sqrt{\sum_{i=0}^{k}\lambda_{i}g_{i}^{2}}}\circ g_{i}=x_{0}-\frac{\gamma_{i}}{\sqrt{t}}\frac{\sqrt{\sum_{i=0}^{k}\left(i+1\right)^{1/2}}}{\sqrt{\sum_{i=0}^{k}\left(i+1\right)^{1/2}g_{i}^{2}}}\circ g_{i}.
\]
This numerator is proportional to $t^{1/4}$. To adapt this sequence
to dual averaging, we must include a step size parameter in the weights.
It's unclear exactly how to do this in a way that maintains the effective
step size property, since if $\lambda_{i}\propto\gamma_{i}$ then
the step size will cancel between the numerator and denominator. 

\subsubsection{MADGRAD's Cube-root denominator}

To maintain the correct effective step size we propose the use of a cube root
instead:
\begin{equation}
x_{k+1}=x_{0}-\frac{1}{\sqrt[3]{\sum_{i=0}^{k}\left(i+1\right)^{1/2}\gamma_{i}g_{i}^{2}}}\circ\sum_{i=0}^{k}\left(i+1\right)^{1/2}\gamma_{i}g_{i}.\label{eq:madgrad_cube}
\end{equation}
Although this modification appears ad-hoc, the use of a cube root
here can actually be motivated by a similar argument used to motivate
the standard square-root formulation. \citet{duchi2011adaptive} consider
the following minimization problem over a $D$ dimensional vector
$s$:
\[
\min_{s}\sum_{i=0}^{k}\sum_{d=0}^{D}\frac{g_{id}^{2}}{s_{d}},\;\left\langle 1,s\right\rangle \leq c, \;\forall d:\,s_{d}>0,
\]
which is solved by $s_{d}\propto\sqrt{\sum_{i=0}^{k}g_{id}^{2}}$.
The motivation for this surrogate problem is to minimize weighted
square norm of the gradients in hind-sight. Rather than a linear penalty
on the size of $s$, which when combined with the positivity constraint
is just a L1 norm penalty $\left\Vert s\right\Vert _{1}\leq c$, if
we instead use a L2 norm penalty:
\[
\min_{s}\sum_{i=0}^{k}\sum_{d=0}^{D}\frac{g_{id}^{2}}{s_{d}},\;\left\Vert s\right\Vert _{2}^{2}\leq c,\;\forall d:\,s_{d}>0
\]
then we recover a cube-root solution $s_{d}\propto\sqrt[3]{\sum_{i=0}^{k}g_{id}^{2}}$. We show this in the Appendix.
The cube root maintains the effective step size as can be seem by
considering that $\sum_{i}^{k}\left(i+1\right)^{1/2}\propto\left(k+1\right)^{3/2}$
which after the cube root operation gives the necessary $\sqrt{k+1}$
scaled denominator required to cancel against $\lambda$'s square-root
growth.

One disadvantage of this weighting is that it results in a final convergence
rate bound that is not fully adaptive in the sense that the choices
of global step size will depend on an expression involving the gradient
norms. We don't believe this is a significant problem given that the
choice of step size still depends on other unknown quantities even
when using a fully adaptive sequence such as the function sub-optimality
gap and gradient bound $G$.

\section{Convergence Theory}
\begin{algorithm}[t]
\begin{algorithmic}[1]
\Require{$\gamma_k$ stepsize sequence, $c_k$ momentum sequence, initial point $x_0$, epsilon $\epsilon$}
\State{$s_0 : d = 0$, $\nu_0 : d = 0$}
\For{$k=0,\dots,T$}
\State{Sample $\xi_k$ and set $g_k=\nabla f(x_k, \xi_k)$}
\State{$\lambda_{k} = \gamma_k \sqrt{k+1}$}
\State{$s_{k+1} = s_k + \lambda_k g_k $}
\State{$\nu_{k+1} = \nu_k + \lambda_k \left( g_{k} \circ g_k \right)$}
\State{$$z_{k+1} = x_0 - \frac{1}{\sqrt[3]{\nu_{k+1}}+\epsilon} \circ s_{k+1}$$}
\State{$x_{k+1}=\left(1-c_{k+1}\right)x_{k}+c_{k+1}z_{k+1}.$}
\EndFor
\State\Return $x_T$
\end{algorithmic}
\caption{\label{alg:MADGRAD}MADGRAD}
\end{algorithm}
\label{sec:theory} The MADGRAD algorithm, combining the discussed ideas, is listed in
Algorithm \ref{alg:MADGRAD}. In order to establish convergence results for
potentially non-smooth functions, we rely on a bounded gradient assumption:
\[
\left\Vert \nabla f(x,\xi)\right\Vert _{\infty}\leq G\;\text{for all }x,\xi.
\]
We also assume each $f(\cdot,\cdot)$ is proper and convex in $x$ over all
$\mathbb{R}^{D}$. Our analysis uses a slight variant of Algorithm
\ref{alg:MADGRAD}, where the denominator includes an extra term $\lambda_{k}G^{2}$:
\begin{equation}
z_{k+1}=x_{0}-\frac{1}{\sqrt[3]{\lambda_{k+1}G^{2}+v_{k+1}}} \circ s_{k+1},\label{eq:madgrad_mod}
\end{equation}
A similar term is also needed by the original DA-AdaGrad method in
\citet{duchi2011adaptive}, and appears necessary for bounding the
accumulated error. We don't believe this term plays an important role
in practice as its magnitude quickly diminishes, and so we have not
included this term in Algorithm \ref{alg:MADGRAD}. A per-coordinate
upper bound $G_{d}$ may be used instead of $G$ to further tighten the theory.
\begin{thm}
After $k$ steps of MADGRAD using the update in Equation~\ref{eq:madgrad_mod},
\begin{align*}
\mathbb{E}\left[f(x_{k})-f(x_{*})\right] & \leq\frac{6}{k^{1/2}}\left\Vert x_{0}-x_{*}\right\Vert _{2}GD^{1/2},
\end{align*}
if $c_{k}=\frac{3/2}{k+3/2}$ and 
\[
\gamma=\frac{1}{k^{3/4}D^{3/4}G^{1/2}}\left\Vert x_{0}-x_{*}\right\Vert _{2}^{3/2}.
\]
\end{thm}
This bound is very loose. It results from  the application of $\nabla f(x,\xi)_{i}\leq G$
to bound each index of the gradient at each time-step separately,
which does not capture any of the adaptivity of the convergence rate. We discuss more precise bounds below.
Note that $\left\Vert g\right\Vert _{2}\leq D^{1/2}\left\Vert g\right\Vert _{\infty}=G D^{1/2}$,
so the dependence on dimensionality here is comparable to bounds established for non-adaptive stochastic methods which have bounds on the 2-norm of the gradient
on the right instead. Note also that we recommend using a flat $c_{k}=c$
momentum for non-convex problems, this decaying rate is only optimal
in the convex case. A value of $c=0.1$ corresponds to the $\beta=0.9$ momentum commonly used with SGD and Adam. 


\subsection{Adaptivity}
To understand the adaptivity of the method at a more granular level, we can express the convergence
rate as:
\begin{align*}
\mathbb{E}\left[f(x_{k})-f(x_{*})\right] & \leq\frac{3}{\gamma}\frac{1}{\left(k+1\right)^{3/2}}\sum_{d=0}^{D}\left(\mathbb{E}\left[\lambda_{k}\left(\sum_{i=0}^{k}\lambda_{i}g_{id}^{2}\right)^{2/3}\right]\right)\\
 & +\frac{3}{\gamma}\frac{1}{\left(k+1\right)^{3/2}}\sum_{d=0}^{D}\left(x_{0x}-x_{*d}\right)^{2}\mathbb{E}\left(\lambda_{k+1}G^{2}+\sum_{i=0}^{k}\lambda_{i}g_{id}^{2}\right)^{1/3}
\end{align*}
The convergence rate heavily depends on a weighted sequence:
\[
\sum_{d=0}^{D}\sum_{i=0}^{k}\lambda_{i}g_{id}^{2}=\gamma\sum_{d=0}^{D}\sum_{i=0}^{k}\left(i+1\right)^{1/2}g_{id}^{2},
\]
rather than an unweighted sum $\sum_{d=0}^{D}\sum_{i=0}^{k}g_{id}^{2}$ used in AdaGrad.
This is key to understanding the performance characteristics of MADGRAD
over traditional AdaGrad. In particular, large gradients at the early
stages have a smaller effect on the overall bound then they do for
AdaGrad. This can be quantified by considering the behavior when the
gradient norm bound is time dependent, i.e. $\left\Vert \nabla f(x_{i},\xi)\right\Vert _{\infty}\leq G_{i}$.
Then as we show in the appendix, for MADGRAD, when using optimal step-sizes:

\begin{align*}
\mathbb{E}\left[f(x_{k})-f(x_{*})\right] & \leq\frac{6}{\left(k+1\right)^{5/4}}\left\Vert x_{0}-x_{*}\right\Vert _{2} D^{1/2} \left(\sum_{i=0}^{k}\left(i+1\right)^{1/2}G_{i}^{2}\right)^{1/2},
\end{align*}
whereas for AdaGrad with the use of momentum:
\begin{align*}
\mathbb{E}\left[f(x_{k})-f(x_{*})\right] & \leq\frac{6}{\left(k+1\right)}\left\Vert x_{0}-x_{*}\right\Vert _{2}D^{1/2}\left(\sum_{i=0}^{k}G_{i}^{2}\right)^{1/2}.
\end{align*}
In MADGRAD the effect of an ``outlier'' $G_{i}$ that is particularly
large at time-step $i$ decays at a faster rate, with a power $5/4$
compared to linearly for AdaGrad. Using $\lambda_{i}$ with larger
power than $1/2$ is also possible within our momentumized-dual averaged gradient framework, which would result in a faster decay. We have found that the 1/2 factor is a "Sweet-spot", as larger values result in empirically slower convergence. Similar convergence rate bounds can be derived using the same proof technique, although they are prefixed by progressively larger constants (growing factorially in the power) as the power used is increased. In general, the advantage of MADGRAD over AdaGrad manifests in the common situation where the gradients are largest at the early stages of optimization.

\subsection{Comparison to Adam}
Although Adam is known to potentially diverge, we can consider the theoretical properties of the
AMSGrad variant of Adam, which is perhaps the smallest modification
to Adam that results in provable convergence. For AMSGrad, parameterized
by momentum $\beta_{1}\lambda^{i-1}$ at step i, assuming a bounded
domain with $R=\max_{x,y}\left\Vert x-y\right\Vert _{\infty}^{2}$,
defining $\gamma=\beta_{1}/\sqrt{\beta_{2}}$, and using step size
$\alpha_{i}=\alpha/\sqrt{i}$ \citep{reddibeyondadam2018}:
\begin{align*}
\mathbb{E}\sum_{i=1}^{k}f(x_{i})-f(x_{*}) & \leq\frac{\beta_{1}RG}{\left(1-\beta_{1}\right)^{2}\left(1-\lambda\right)^{2}}+\frac{R\sqrt{T}}{\alpha\left(1-\beta_{1}\right)}\sum_{d=1}^{D}\left(\hat{v}_{k,d}\right)^{1/2}\\
 & +\frac{\alpha\sqrt{1+\log k}}{\left(1-\beta_{1}\right)^{2}\left(1-\gamma\right)\sqrt{1-\beta_{2}}}\sum_{d}^{D}\left(\sum_{i=1}^{k}g_{id}^{2}\right)^{1/2}
\end{align*}
$\hat{v}$ is the maximum of the exponential moving average of the squared gradients, see \cite{reddibeyondadam2018} for further details. This result has a number of shortcomings compared to the MADGRAD. Firstly, note that the momentum term $1-\beta_1$, comparable to $c$ in MADGRAD divides each term in the bound. This means that momentum hurts rather than improves performance. The dependence on a bounded domain is also an undesirable property compared to MADGRAD, and the convergence theory of MADGRAD avoids log factors.

\section{Experimental Results}

\begin{figure}
\includegraphics[width=0.49\textwidth]{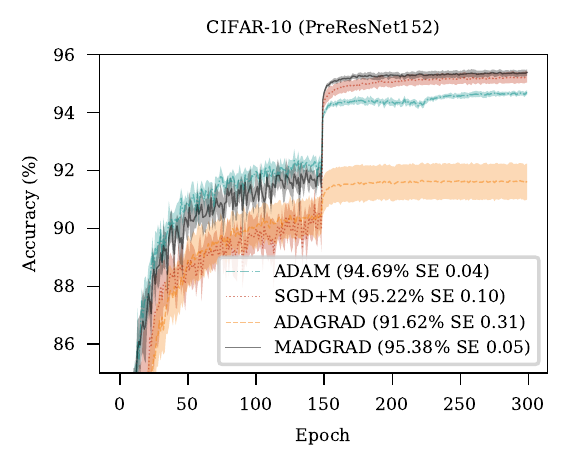}\includegraphics[width=0.49\textwidth]{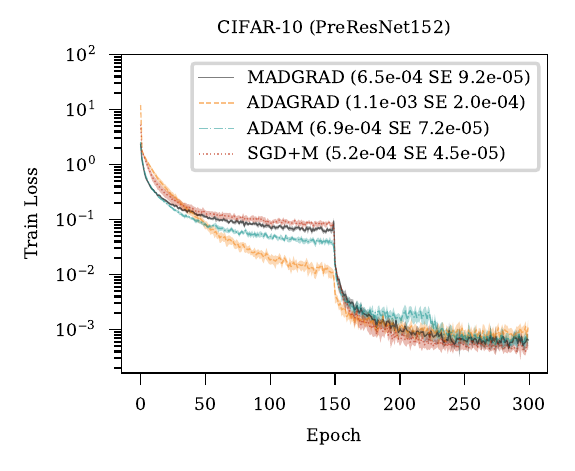}

\includegraphics[width=0.49\textwidth]{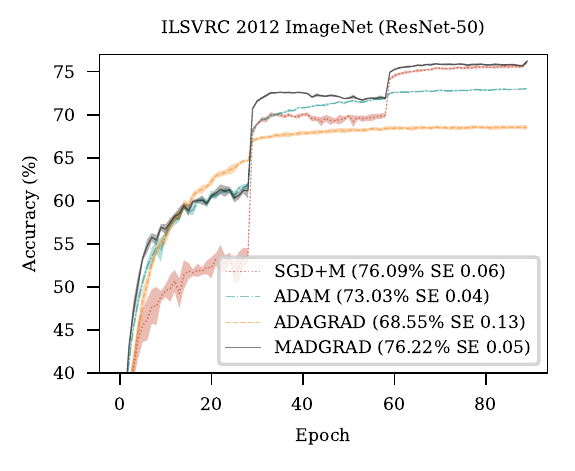}\includegraphics[width=0.49\textwidth]{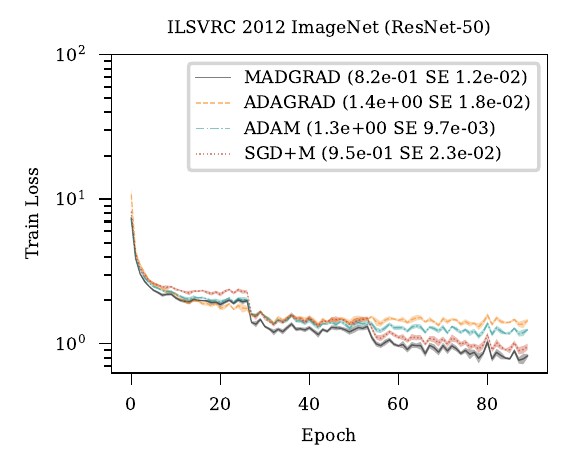}

\includegraphics[width=0.49\textwidth]{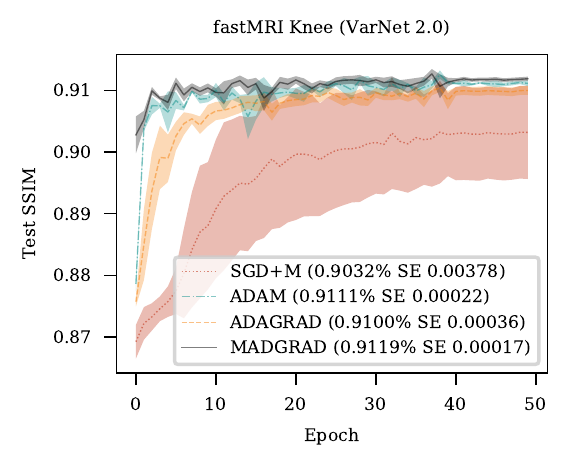}\includegraphics[width=0.49\textwidth]{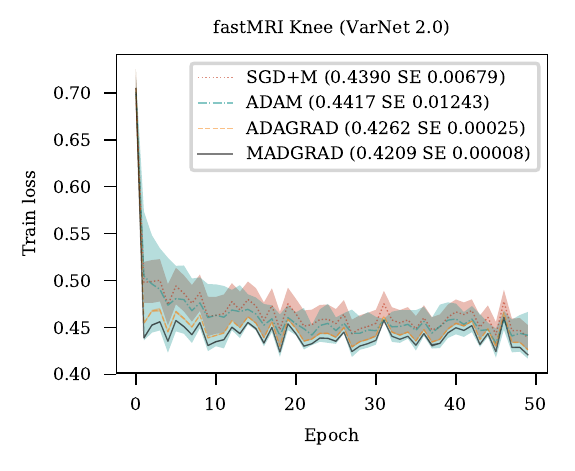}

\caption{\label{fig:first_figure}Experimental results for the CIFAR-10, ImageNet
and fastMRI Knee problems. Left column shows test set performance
and the right column shows training set performance.}
\end{figure}
\label{sec:experiments}In our experiments we compared MADGRAD against
SGD, Adam and AdaGrad. SGD is known to perform well on computer
vision classification problems due to its ability to produce solutions
that generalize better than adaptive methods. In contrast, Adam is
the method of choice in other domains with structured output where
overfitting is less of an issue. We present results across a large
number of problems across both categories to validate the general
purpose utility of the MADGRAD approach.

In our experiments we use the most common step-size reduction scheme used
in the literature for each respective problem. For all algorithms,
we performed a learning rate and decay sweep on a grid on intervals
of $[1\times10^{i},2.5\times10^{i},5\times10^{i}]$ for a range of
$i$ large enough to ensure the best parameters for each problem and
method were considered. We present the results from the best learning
rate and decay for each method when considering test set performance.
For other hyper-parameters, we used commonly accepted defaults for
each respective problem. Full parameter settings used for each method
are listed in the appendix. All presented results are averaged over
a number of seeds with error bars indicating 2 standard errors. Ten
seeds were used for CIFAR-10 and IWSLT14, whereas only five seeds
were used for the remaining larger scale problems.

\subsection{CIFAR10 image classification}

CIFAR10 \citep{cifar} is an established baseline method within the
deep learning community due to its manageable size and representative
performance within the class of data-limited supervised image classification
problems. It is particularly notable for showing clear differences
between adaptive and non-adaptive methods, as the former tend to overfit
considerably on this problem. Following standard practice, we apply
a data-augmentation step consisting of random horizontal flipping,
4px padding followed by random cropping to 32px at training time only.
We used a high-performance pre-activation ResNet architecture \citep{preact}
which is known to work well on this problem, consisting of 58,144,842
parameters across 152 layers. The depth of this network is representative
of the typical point of diminishing returns for network depth on computer
vision problems. As this network is greatly over-parameterized, each
method can be expected to fit the training data exactly, achieving
near zero loss, even with this data augmentation. For this reason,
this task is particularly sensitive to difference in generalization
performance of each method.

As illustrated in Figure \ref{fig:first_figure}, both Adam and AdaGrad
perform poorly on this problem in terms of test accuracy. The under-performance of Adam on this problem is well known \citep{marginal_adaptive2017},
and is typically attributed to convergence to poor local minima, as
the training set convergence is very rapid initially. 

MADGRAD exhibits excellent test accuracy results on this problem, achieving
the highest test accuracy among the methods considered. This demonstrates
that unlike Adam and AdaGrad, MADGRAD's adaptivity does not come at
the cost of inferior generalization performance.

\subsection{ILSVRC 2012 ImageNet image classification}

The ImageNet problem \citep{krizhevsky2012imagenet} is a larger problem
more representative of image classification problems encountered in
industrial applications where a large number of classes and higher
resolution input images are encountered. Like CIFAR10, overfitting
can be an issue on this problem for adaptive methods. We ran experiments
using the ResNet-50 architecture, which is considered the standard
baseline for this problem. This combination of data set and architecture
are one of the most studied in all of machine learning, which makes
it an ideal testing ground for optimization algorithms. 

Our setup used data preprocessing consisting of a mean {[}0.485, 0.456,
0.406{]} and std {[}0.229, 0.224, 0.225{]} normalization of the three
respective color channels, followed by a RandomResizedCrop PyTorch
operation to reduce the resolution to 224 pixels followed by a random
50\% chance of horizontal flipping. For test set evaluation a resize
to 256 pixels followed by a center crop to 224 pixels is used instead.
This setup was used as it is standard within the PyTorch community,
however it differs from the setup in \citet{he2016deep}, meaning
that test accuracy is close but not directly comparable.

On this problem both Adam and AdaGrad show similar convergence properties
as were seen on the CIFAR-10 problem. They both greatly under-perform
SGD with momentum. MADGRAD shows strong performance here as well,
achieving higher test accuracy than any other method for the majority
of the training time, and yielding the best final test accuracy. The
accuracy of MADGRAD at epoch 70 is 75.87, a level only reached by
SGD+M after the learning rate reduction at epoch 90, more than 28\%
longer. MADGRAD also performs the best on training loss on this problem. 

\subsection{fastMRI challenge MRI reconstruction}

The fastMRI Knee challenge \citep{zbontar2018fastmri} is a recently
proposed large-scale image-2-image problem. Unlike the previously
explored classification problems, the scale of this problem makes
overfitting a non-concern given the number of weights in the largest
models currently trainable on contemporary hardware, meaning that
adaptive methods are not prone to overfitting. This problem is also
particularly notable for being poorly conditioned among
image processing problems. Part of the reason for the poor conditioning
is the high depth of current SOTA models, such as the VarNet 2.0 \citet{sriram2020end}
model that we used. This model has 12,931,532 parameters over 273
layers. Our implementation uses 16 auto-calibration lines and an offset
equispaced sampling pattern \citep{defazio2019offset}, which is much
closer to a realistic clinical configuration than the challenge's
random sampling mask.

Figure \ref{fig:first_figure} shows a number of interesting properties
of the methods. SGD+M exhibits extremely variable performance on this problem, and under-performs other methods by a large margin. AdaGrad
also has a clear performance gap compared to the top performing methods,
MADGRAD and Adam. MADGRAD is the best performer, with a small but
statistically significant improvement over Adam, which is the standard
method for this problem. Training set performance shows a much higher
degree of variability, making comparisons difficult, however MADGRAD
appears to also be the best performing method on training loss as
well.

\subsection{Machine translation with a recurrent neural network}

\begin{figure}
\includegraphics[width=0.49\textwidth]{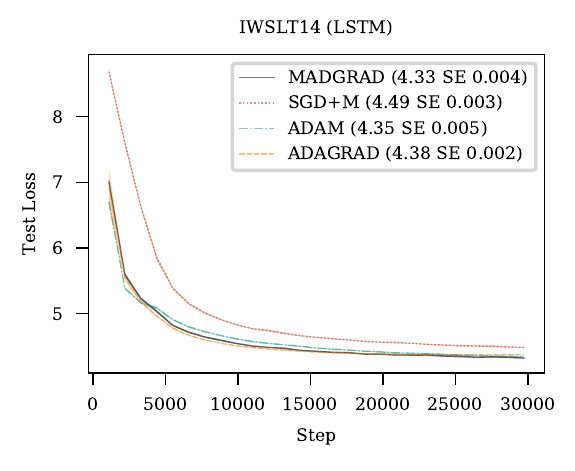}\includegraphics[width=0.49\textwidth]{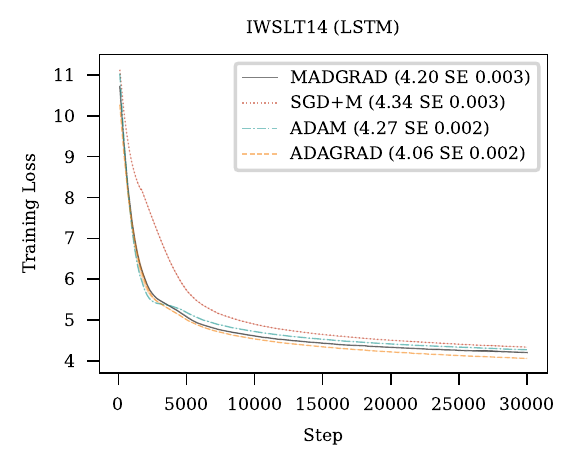}

\includegraphics[width=0.49\textwidth]{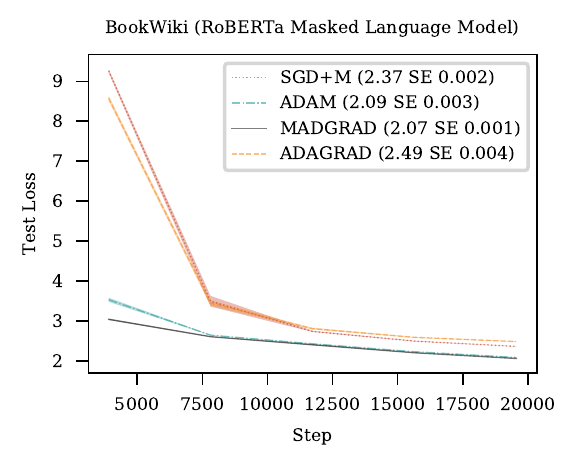}\includegraphics[width=0.49\textwidth]{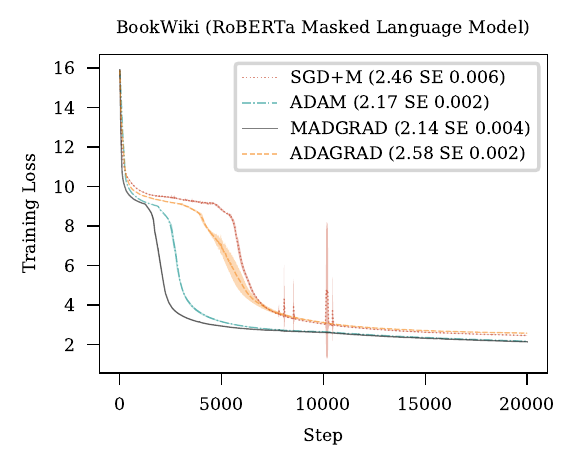}

\caption{\label{fig:second_figure}Experimental results for the IWSLT14 and
BookWiki problems. Left column shows test set performance and the
right column shows training set performance.}
\end{figure}
For a machine translation baseline we trained our model on the IWSLT14
Germain-to-English dataset \citep{cettolo2014report}, using a popular
LSTM variant introduced by \citet{wiseman-rush-2016-sequence}.

Figure \ref{fig:second_figure} shows that all of the adaptive methods
out-perform SGD on this problem by a significant margin. The results
are close but MADGRAD has a small performance lead, yielding 4.33
test loss compared to 4.38 for AdaGrad and 4.35 for Adam. In training
loss AdaGrad's lead over the other methods can be attributed to a
slight degree of overfitting; these is a slight increase in test loss
near the end of optimization for AdaGrad which indicates this.

\subsection{Masked language modeling with a Transformer}

Bidirectional training objectives, as used in the BERT approach
\citep{BERT}, have quickly established themselves as the new standard
for large-scale pre-training of natural language models. We performed
our experiments using the RoBERTa variant of BERT\_BASE \citep{liu2019roberta},
a 110M parameter transformer model. This model is large enough to
provide a realistic optimization test-bed for large-scale Transformer
models while still being trainable in in time comparable to a ResNet-50
model on ImageNet. 

Similar to the LSTM problem, SGD+M performs poorly here. It exhibits
some spikes where training loss rapidly degrades then recovers quickly.
both Adam and MADGRAD perform well, however MADGRAD is significantly
faster initially, and also achieves a better final test loss of 2.07
compared to 2.09 achieved by Adam.

\section{Discussion}
\subsection{Hyper-parameter settings}
We have made the following observations during our experimentation:
\begin{itemize}
    \item Typically, using the default weight decay from previous SGD/Adam training runs will result in poor generalization performance. Weight decay will need to be much less, potentially even 0, for good performance. We recommend reducing the weight-decay before any learning rate tuning.
    \item Learning rate values are not directly comparable to SGD/Adam, a full learning rate sweep is necessary to find the optimal value. In the appendix we list the best LR values for each of our test problems, which should form a good starting point. Sweeping across a power-of-2 grid is recommended as the value several of orders of magnitude different from SGD/Adam.
    \item Momentum values used for SGD/Adam should work without issue, by setting $c=1-\beta$ for momentum $\beta$.
\end{itemize}

\subsection{Empirical results in deep learning}
We believe our experimental validation is one of the most comprehensive
performed for any newly proposed deep learning optimization method.
More than 20,000 hours of GPU time were needed to perform the grid
search and final evaluation mentioned above, as we performed the search
for each of the methods considered, rather than just the MADGRAD method.
This prevents our method looking better than it would otherwise look
due to hyper-parameter optimization rather than an actual performance
advantage. Our comparison also includes a number of large and realistic
problems, which are better representative of modern deep learning
compared to small scale problems.  Finally, our final results are
averaged over a sufficiently large number of seeds for each problem
to ensure that run-to-run variation is not mistaken for actual performance
differences. This is particularly a problem with CIFAR-10, yet many
published results still use only a single seed for comparisons on
that problem. For these reasons, we believe our experimental results
for MADGRAD are representative of the performance of the method across
modern large-scale empirical risk minimization problems.

\subsection{Sparsity}
The reliance on a slowly updating moving average for the squared gradient within the Adam method greatly hinders its application to sparse models. In contrast, MADGRAD maintains a simple sum of the squared gradient entries which may be updated in a sparse fashion. One potential problem in the sparse case is that the buffer of iterates (rather than gradients) is maintained with a moving average. To support sparse applications, the iterate buffer may be removed, effectively equivalent to setting $c=1$. 

\section{Conclusion}
We have proposed the MADGRAD (Momentumized, Adaptive, Dual averaged GRADient) method as a general purpose optimizer for deep learning. Given MADGRAD's state-of-the-art empirical performance, together with its strong theoretical foundations, it is an excellent first choice of optimizer across many sub-fields of machine learning. 

\bibliographystyle{plainnat}
\bibliography{madgrad}

\appendix

\section{Parameter settings}

\subsection*{CIFAR10}

Our data augmentation pipeline followed standard practice: random
horizontal flipping, then random cropping to 32x32, then normalization
by centering around (0.5, 0.5, 0.5).

\begin{tabular}{|c|c|}
\hline 
Hyper-parameter  & Value\tabularnewline
\hline 
\hline 
Architecture  & PreAct ResNet152\tabularnewline
\hline 
Epochs  & 300\tabularnewline
\hline 
GPUs  & 1xV100\tabularnewline
\hline 
Batch Size per GPU  & 128\tabularnewline
\hline 
LR schedule & 150-225 tenthing\tabularnewline
\hline 
Seeds & 10\tabularnewline
\hline 
\end{tabular}

\begin{tabular}{|c|c|c|}
\hline 
Method & LR & Decay\tabularnewline
\hline 
\hline 
MADGRAD & 2.5e-4 & 0.0001\tabularnewline
\hline 
AdaGrad & 0.01 & 0.0001\tabularnewline
\hline 
Adam & 0.00025 & 0.0001\tabularnewline
\hline 
SGD & 0.1 & 0.0001\tabularnewline
\hline 
\end{tabular}

\subsection*{ImageNet}

A standard LR schedule was used, where the learning rate is decreased
10 fold every 30 epochs. Interestingly, for this problem, a smaller
decay constant improved the performance of MADGRAD, but didn't yield
any improvement to the other methods considered.

\begin{tabular}{|c|c|}
\hline 
Hyper-parameter  & Value\tabularnewline
\hline 
\hline 
Architecture  & ResNet50\tabularnewline
\hline 
Epochs  & 90\tabularnewline
\hline 
GPUs  & 8xV100\tabularnewline
\hline 
Batch size per GPU  & 32\tabularnewline
\hline 
LR schedule & 30-60-90 tenthing\tabularnewline
\hline 
Seeds & 5\tabularnewline
\hline 
\end{tabular}

\begin{tabular}{|c|c|c|}
\hline 
Method & LR & Decay\tabularnewline
\hline 
\hline 
MADGRAD & 0.001 & 2.5e-5\tabularnewline
\hline 
AdaGrad & 0.01 & 0.0001\tabularnewline
\hline 
Adam & 0.00025 & 0.0001\tabularnewline
\hline 
SGD & 0.1 & 0.0001\tabularnewline
\hline 
\end{tabular}

\subsection*{fastMRI}

For this task, the best learning rate schedule is a flat schedule,
with a small number of fine-tuning epochs at the end to stabilize.
To this end, we decreased the learning rate 10 fold at epoch 40. 

\begin{tabular}{|c|c|}
\hline 
Hyper-parameter  & Value\tabularnewline
\hline 
\hline 
Architecture  & 12 layer VarNet 2\tabularnewline
\hline 
Epochs  & 50\tabularnewline
\hline 
GPUs  & 8xV100\tabularnewline
\hline 
Batch size per GPU  & 1\tabularnewline
\hline 
Acceleration factor  & 4\tabularnewline
\hline 
Low frequency lines  & 16\tabularnewline
\hline 
Mask type  & Offset-1\tabularnewline
\hline 
LR schedule & 40 tenthing\tabularnewline
\hline 
Seeds & 5\tabularnewline
\hline 
\end{tabular}

\begin{tabular}{|c|c|c|}
\hline 
Method & LR & Decay\tabularnewline
\hline 
\hline 
MADGRAD & 0.01 & 0.0\tabularnewline
\hline 
AdaGrad & 0.25 & 0.0\tabularnewline
\hline 
Adam & 0.00025 & 0.0\tabularnewline
\hline 
SGD & 0.01 & 0.0\tabularnewline
\hline 
\end{tabular}

\subsection*{IWSLT14}

Our implementation used FairSeq defaults except for the parameters
listed below. 

\begin{tabular}{|c|c|}
\hline 
Hyper-parameter  & Value\tabularnewline
\hline 
\hline 
Architecture  & lstm\_wiseman\_iwslt\_de\_en\tabularnewline
\hline 
Max updates  & 60,000\tabularnewline
\hline 
GPUs  & 1xV100\tabularnewline
\hline 
Max tokens per batch  & 4096\tabularnewline
\hline 
Warmup steps  & 4000\tabularnewline
\hline 
Dropout  & 0.3\tabularnewline
\hline 
Label smoothing  & 0.1\tabularnewline
\hline 
Share decoder/input/output embed  & True\tabularnewline
\hline 
Float16  & True\tabularnewline
\hline 
Update Frequency  & 1\tabularnewline
\hline 
LR schedule & Inverse square-root\tabularnewline
\hline 
Seeds & 10\tabularnewline
\hline 
\end{tabular}

\begin{tabular}{|c|c|c|}
\hline 
Method & LR & Decay\tabularnewline
\hline 
\hline 
MADGRAD & 0.025 & 5e-6\tabularnewline
\hline 
AdaGrad & 0.25 & 1e-5\tabularnewline
\hline 
Adam & 0.01 & 0.05\tabularnewline
\hline 
SGD & 1.0 & 1e-5\tabularnewline
\hline 
\end{tabular}

\subsection*{BookWiki}

Our implementation used FairSeq defaults except for the parameters
listed below.

\begin{tabular}{|c|c|}
\hline 
Hyper-parameter  & Value\tabularnewline
\hline 
\hline 
Architecture  & roberta\_base\tabularnewline
\hline 
Task  & masked\_lm\tabularnewline
\hline 
Max updates  & 20,000\tabularnewline
\hline 
GPUs  & 8xV100\tabularnewline
\hline 
Max tokens per sample  & 512\tabularnewline
\hline 
Dropout & 0.1\tabularnewline
\hline 
Attention Dropout & 0.1\tabularnewline
\hline 
Max sentences & 16\tabularnewline
\hline 
Warmup  & 10,000\tabularnewline
\hline 
Sample Break Mode  & Complete\tabularnewline
\hline 
Share decoder/input/output embed  & True\tabularnewline
\hline 
Float16  & True\tabularnewline
\hline 
Update Frequency  & 16\tabularnewline
\hline 
LR schedule & Polynomial decay\tabularnewline
\hline 
Seeds & 5\tabularnewline
\hline 
Gradient clipping & 0.5\tabularnewline
\hline 
\end{tabular}

\begin{tabular}{|c|c|c|}
\hline 
Method & LR & Decay\tabularnewline
\hline 
\hline 
MADGRAD & 0.005 & 0.0\tabularnewline
\hline 
AdaGrad & 0.01 & 0.0\tabularnewline
\hline 
Adam & 0.001 & 0.0\tabularnewline
\hline 
SGD & 1.0 & 0.0\tabularnewline
\hline 
\end{tabular}

\section{Theory}

\subsection{Theoretical variant}
We analyze a variant of the MADGRAD algorithm, using fixed step size
$\gamma$, and $\lambda_{k} = \gamma\sqrt{k+1}$:
\begin{align}
s_{k+1} & =s_{k}+\lambda_{k}g_{k},\nonumber \\
v_{k+1} & =v_{k}+\lambda_{k}g_{k}^{2},\nonumber \\
z_{k+1} & =x_{0}-\frac{1}{\sqrt[3]{\lambda_{k+1}G^{2}+v_{k+1}}}s_{k+1},\nonumber \\
x_{k+1} & =\left(1-c_{k+1}\right)x_{k}+c_{k+1}z_{k+1}.\label{eq:madgradv}
\end{align}
This variant differs from Algorithm \ref{alg:MADGRAD} just with the
addition of $\lambda_{k}G^{2}$ in the denominator, which is necessitated
by our analysis method. Note that the AdaGrad DA formulation originally
proposed by \citet{duchi2011adaptive} also requires this extra term.

\subsection{Support function}
We define a matrix analogue of the support function from \citet{nesterov2009primal}:
\begin{equation}
V_{A_{k}}(-s_{k})=\max_{x}\left\{ -\left\langle s_{k},x-x_{0}\right\rangle -\frac{1}{2}\left\Vert x-x_{0}\right\Vert _{A_{k}}^{2}\right\} .\label{eq:value_function}
\end{equation}
In this work we only consider diagonal $A_{k}$, represented by a
vector $a_{k}:$
\[
A_{k}=\text{diag}(a_{k}).
\]
In this notation, we have $\alpha_{k}=\sqrt[3]{\lambda_{k} G^2 + v_{k}}$. The maximizer of expression \ref{eq:value_function} is (using component-wise
division): 
\[
z_{k}=x_{0}-\frac{s_{k}}{\alpha_{k}}.
\]
Since
$v_{k+1}$ is non-decreasing, it's clear that:
\begin{equation}
V_{A_{k+1}}\left(-s_{k}\right)\leq V_{A_{k}}\left(-s_{k}\right).\label{eq:V_decrease}
\end{equation}
We will also use the following properties, which follow directly by
modifying the argument in \citet{nesterov2009primal} to handle scaling
matrices instead of constants:
\begin{equation}
\nabla V_{A_{k}}(-s_{k})=z_{k}-x_{0},\label{eq:v-grad}
\end{equation}
\begin{equation}
V_{A_{k}}(s+\delta)\leq V_{A_{k}}(s)+\left\langle \delta,\nabla V_{A_{k}}(s)\right\rangle +\frac{1}{2}\left\Vert \delta\right\Vert _{A_{k}^{-1}}^{2}.\label{eq:v-l-smooth}
\end{equation}

\subsection{Lemmas}
\begin{lem}
\label{lem:error_sum_bound}For all natural $k$, assuming $\lambda_{k+1}\geq\lambda_{k}$:
\[
\sum_{t=0}^{k}\frac{\lambda_{t}^{2}g_{t}^{2}}{\left(\lambda_{t}G^{2}+\sum_{i=0}^{t-1}\lambda_{i}g_{i}^{2}\right)^{1/3}}\leq\frac{3}{2}\lambda_{k}\left(\sum_{i=0}^{k}\lambda_{i}g_{i}^{2}\right)^{2/3}.
\]
\end{lem}
\begin{proof}
We prove by induction. For the base case:
\[
\frac{g_{0}^{2}}{\left(G^{2}\right)^{1/3}}\leq g^{2(1-1/3)}=\left(g^{2}\right)^{2/3}\leq\frac{3}{2}\left(g^{2}\right)^{2/3}.
\]
Now assume the lemma holds for $k-1$ then using the inductive hypothesis
\begin{align*}
\sum_{t=0}^{k}\frac{\lambda_{t}^{2}g_{t}^{2}}{\left(\lambda_{t}G^{2}+\sum_{i=0}^{t-1}\lambda_{i}g_{i}^{2}\right)^{1/3}} & \leq\frac{\lambda_{k}^{2}g_{k}^{2}}{\left(\lambda_{t}G^{2}+\sum_{i=0}^{k-1}\lambda_{i}g_{i}^{2}\right)^{1/3}}+\frac{3}{2}\lambda_{k-1}\left(\sum_{i=0}^{k-1}\lambda_{i}g_{i}^{2}\right)^{2/3},\\
 & \leq\frac{\lambda_{k}^{2}g_{k}^{2}}{\left(\lambda_{t}G^{2}+\sum_{i=0}^{k-1}\lambda_{i}g_{i}^{2}\right)^{1/3}}+\frac{3}{2}\lambda_{k}\left(\sum_{i=0}^{k-1}\lambda_{i}g_{i}^{2}\right)^{2/3}.
\end{align*}
Define $b_{k}=\sum_{i=0}^{k}\lambda_{i}g_{i}^{2}$ and $a_{k}=g_{k}^{2}$
then we have:
\[
\sum_{t=0}^{k}\frac{\lambda_{t}^{2}g_{t}^{2}}{\left(\lambda_{t}G^{2}+\sum_{i=0}^{t-1}\lambda_{i}g_{i}^{2}\right)^{1/3}}\leq\lambda_{k}^{2}a_{k}\left(\lambda_{k}G^{2}+b_{k}-\lambda_{k}a_{k}\right)^{-1/3}+\frac{3}{2}\lambda_{k}\left(b_{k}-\lambda_{k}a_{k}\right)^{2/3}.
\]
We have two terms on the right to consider. For the first term, note
that since $a_{k}\leq G^{2}$, 
\[
\lambda_{k}^{2}a_{k}\left(\lambda_{k}G^{2}+b_{k}-\lambda_{k}a_{k}\right)^{-1/3}\leq\lambda_{k}^{2}a_{k}\left(b_{k}\right)^{-1/3}.
\]
For the 2nd term, we can use concavity to get:
\[
\frac{3}{2}\lambda_{k}\left(b_{k}-\lambda_{k}a_{k}\right)^{2/3}\leq\frac{3}{2}\lambda_{k}\left(b_{k}\right)^{2/3}-\lambda_{k}^{2}a_{k}\left(b_{k}\right)^{-1/3}.
\]
Combining gives:
\[
\sum_{t=0}^{k}\frac{\lambda_{t}^{2}g_{t}^{2}}{\left(\lambda_{t}G^{2}+\sum_{i=0}^{t-1}\lambda_{i}g_{i}^{2}\right)^{1/3}}\leq\frac{3}{2}\lambda_{k}\left(b_{k}\right)^{2/3},
\]
and so the inductive case is proven.
\end{proof}
\begin{lem}
\label{lem:ck_iterateweighting} Let $0<r<1$ and $j\geq0$. Then
define:
\end{lem}
\[
c_{k}=\frac{r+1}{k+j+r},
\]
for all $k\geq0$ it then holds that:
\[
\frac{1-c_{k}}{c_{k}}(k+j)^{r}\leq\frac{1}{c_{k-1}}(k+j-1)^{r}.
\]
\begin{proof}
We start by simplifying:
\begin{align*}
\frac{1-c_{k}}{c_{k}}(k+j)^{r} & =\frac{1-\frac{r+1}{k+j+r}}{\frac{r+1}{k+j+r}}(k+j)^{r},\\
 & =\frac{k+j-1}{r+1}(k+j)^{r},\\
 & =\frac{k+j+r-1}{r+1}\frac{k+j-1}{k+j+r-1}(k+j)^{r},\\
 & =\frac{1}{c_{k}}\frac{k+j-1}{k+j+r-1}(k+j)^{r}.
\end{align*}
So we need:
\[
(k+j)^{r}\leq\frac{k+j+r-1}{k+j-1}\left(k+j-1\right)^{r}.
\]
Recall the concavity upper bound:
\[
f(x)\leq f(y)+\left\langle \nabla f(y),x-y\right\rangle ,
\]
using $f(x)=\left(k+j\right)^{r}$ which is concave for $r\in(0,1)$,
and $x=k+j,y=k+j-1,$ we have:
\begin{align*}
\left(k+j\right)^{r} & \leq\left(k+j-1\right)^{r}+r\left(k+j-1\right)^{r-1},\\
 & =\left(k+j-1\right)^{r}+\frac{r}{k+j-1}\left(k+j-1\right)^{r},\\
 & =\frac{k+j-1+r}{k+j-1}\left(k+j-1\right)^{r}.
\end{align*}
Which establishes the result.
\end{proof}
\begin{lem}
The dual averaging iterates obey:
\begin{equation}
z_{k}=x_{k}-\frac{1-c_{k}}{c_{k}}\left(x_{k-1}-x_{k}\right).\label{eq:x-diff}
\end{equation}
\end{lem}
\begin{proof}
We rearrange the $x$ update:
\[
x_{k+1}=\left(1-c_{k+1}\right)x_{k}+c_{k+1}z_{k+1}.
\]
\[
\therefore x_{k}=\left(1-c_{k}\right)x_{k-1}+c_{k}z_{k},
\]
\[
\therefore c_{k}z_{k}=x_{k}-(1-c_{k})x_{k-1},
\]
\[
\therefore z_{k}=\frac{1}{c_{k}}x_{k}-\frac{1-c_{k}}{c_{k}}x_{k-1}.
\]
\end{proof}
\begin{thm}
\label{thm:lyapunov_step}Consider the MADGRAD method. We upper bound
the quantity $V_{A_{k+1}}\left(-s_{k+1}\right)$ as follows:
For the first step $k=0$:
\[
V_{A_{1}}\left(-s_{1}\right)\leq\frac{\lambda_{0}^{2}}{2}\left\Vert \nabla f\left(x_{0},\xi_{k}\right)\right\Vert _{A_{0}^{-1}}^{2}.
\]
For subsequent steps $k\ge1$:
\begin{align*}
V_{A_{k+1}}\left(-s_{k+1}\right) & \leq V_{A_{k}}\left(-s_{k}\right)+\frac{\lambda_{k}^{2}}{2}\left\Vert \nabla f\left(x_{k},\xi_{k}\right)\right\Vert _{A_{k}^{-1}}^{2}+\lambda_{k}\left\langle \nabla f\left(x_{k},\xi_{k}\right),x_{0}-x_{*}\right\rangle \\
 & -\frac{1}{c_{k}}\lambda_{k}\left[f(x_{k},\xi_{k})-f(x_{*},\xi_{k})\right]+\frac{1-c_{k}}{c_{k}}\lambda_{k}\left[f(x_{k-1},\xi_{k})-f(x_{*},\xi_{k})\right].
\end{align*}
\end{thm}
\begin{proof}
Base case:
\begin{align}
V_{A_{1}}\left(-s_{1}\right) & \leq-\lambda_{0}\left\langle \nabla f\left(x_{0},\xi_{k}\right),\nabla V_{0}\left(-s_{0}\right)\right\rangle +\frac{\lambda_{0}^{2}}{2}\left\Vert \nabla f\left(x_{0},\xi_{k}\right)\right\Vert _{A_{0}^{-1}}^{2}\quad\text{(Eq. \ref{eq:v-l-smooth})}\nonumber, \\
 & =\lambda_{k}\left\langle \nabla f\left(x_{k},\xi_{k}\right),x_{0}-x_{0}\right\rangle +\frac{\lambda_{0}^{2}}{2\beta_{0}}\left\Vert \nabla f\left(x_{0},\xi_{k}\right)\right\Vert _{A_{0}^{-1}}^{2},\quad\text{(Eq. \ref{eq:v-grad})}\nonumber \\
 & =\frac{\lambda_{0}^{2}}{2}\left\Vert \nabla f\left(x_{0},\xi_{k}\right)\right\Vert _{A_{0}^{-1}}^{2}.\label{eq:base-case}
\end{align}
Inductive case:
\begin{align*}
V_{A_{k+1}}\left(-s_{k+1}\right) & \leq V_{A_{k}}\left(-s_{k+1}\right)\\
 & \leq V_{A_{k}}\left(-s_{k}\right)-\lambda_{k}\left\langle \nabla f\left(x_{k},\xi_{k}\right),\nabla V_{A_{k}}\left(-s_{k}\right)\right\rangle +\frac{\lambda_{k}^{2}}{2}\left\Vert \nabla f\left(x_{k},\xi_{k}\right)\right\Vert _{A_{k}^{-1}}^{2},\quad\text{(Eq. \ref{eq:v-l-smooth})}\\
 & =V_{A_{k}}\left(-s_{k}\right)+\lambda_{k}\left\langle \nabla f\left(x_{k},\xi_{k}\right),x_{0}-z_{k}\right\rangle +\frac{\lambda_{k}^{2}}{2}\left\Vert \nabla f\left(x_{k},\xi_{k}\right)\right\Vert _{A_{k}^{-1}}^{2},\quad\text{(Eq. \ref{eq:v-grad})}\\
 & =V_{A_{k}}\left(-s_{k}\right)+\frac{\lambda_{k}^{2}}{2}\left\Vert \nabla f\left(x_{k},\xi_{k}\right)\right\Vert _{A_{k}^{-1}}^{2}\\
 & +\lambda_{k}\left\langle \nabla f\left(x_{k},\xi_{k}\right),x_{0}-x_{k}+\left(\frac{1-c_{k}}{c_{k}}\right)\left(x_{k-1}-x_{k}\right)\right\rangle, \quad\text{(Eq. \ref{eq:x-diff})}\\
 & =V_{A_{k+1}}\left(-s_{k}\right)+\frac{\lambda_{k}^{2}}{2}\left\Vert \nabla f\left(x_{k},\xi_{k}\right)\right\Vert _{A_{k}^{-1}}^{2}\\
 & +\lambda_{i}\left\langle \nabla f\left(x_{k},\xi_{k}\right),x_{0}-x_{k}\right\rangle +\lambda_{k}\frac{1-c_{k}}{c_{k}}\left\langle \nabla f\left(x_{k},\xi_{k}\right),x_{k-1}-x_{k}\right\rangle, \\
 & =V_{A_{k+1}}\left(-s_{k}\right)+\frac{\lambda_{k}^{2}}{2}\left\Vert \nabla f\left(x_{k},\xi_{k}\right)\right\Vert _{A_{k}^{-1}}^{2}\\
 & +\lambda_{k}\left\langle \nabla f\left(x_{k},\xi_{k}\right),x_{0}-x_{*}\right\rangle +\lambda_{k}\left\langle \nabla f\left(x_{k},\xi_{k}\right),x_{*}-x_{k}\right\rangle \\
 & +\lambda_{k}\frac{1-c_{k}}{c_{k}}\left\langle \nabla f\left(x_{k},\xi_{k}\right),x_{k-1}-x_{k}\right\rangle.
\end{align*}
Now we use:
\[
\left\langle \nabla f\left(x_{k},\xi_{k}\right),x_{*}-x_{k}\right\rangle \leq f(x_{*},\xi_{k})-f(x_{k},\xi_{k}),
\]
and:
\[
\left\langle \nabla f\left(x_{k},\xi_{k}\right),x_{k-1}-x_{k}\right\rangle \leq f(x_{k-1},\xi_{k})-f(x_{k},\xi_{k}),
\]
to give:
\begin{align*}
V_{A_{k+1}}\left(-s_{k+1}\right) & \leq V_{A_{k}}\left(-s_{k}\right)+\frac{\lambda_{k}^{2}}{2}\left\Vert \nabla f\left(x_{k},\xi_{k}\right)\right\Vert _{A_{k}^{-1}}^{2}\\
 & +\lambda_{k}\left\langle \nabla f\left(x_{k},\xi_{k}\right),x_{0}-x_{*}\right\rangle \\
 & +\lambda_{k}\left[f(x_{*},\xi_{k})-f(x_{k},\xi_{k})\right]+\lambda_{k}\frac{1-c_{k}}{c_{k}}\left[f(x_{k-1},\xi_{k})-f(x_{k},\xi_{k})\right],
\end{align*}
grouping function value terms gives the result.
\end{proof}

\subsection{Convergence rate}
\begin{thm}
After $k$ steps of MADGRAD, 
\begin{align*}
\mathbb{E}\left[f(x_{k})-f(x_{*})\right] & \leq\frac{6}{k^{1/2}}\left\Vert x_{0}-x_{*}\right\Vert GD^{1/2},
\end{align*}
if $c_{k}=\frac{3/2}{k+3/2}$ and
\end{thm}
\[
\gamma=\frac{1}{k^{3/4}D^{3/4}G^{1/2}}\left\Vert x_{0}-x_{*}\right\Vert ^{3/2}.
\]
We assume that $\gamma_{k}=\gamma$ is a constant. First note that
for our choice of $\lambda_{k}=\gamma\left(k+1\right)^{1/2}$ and:
\[
c_{k}=\frac{3/2}{k+3/2},
\]
applying Lemma \ref{lem:ck_iterateweighting} gives that: 
\[
\frac{1-c_{k}}{c_{k}}\lambda_{k}\leq\frac{1}{c_{k-1}}\lambda_{k-1}.
\]
Using this bound we can telescope the bound from Theorem \ref{thm:lyapunov_step}
after taking expectations:
\begin{align*}
\frac{1}{c_{k}}\lambda_{k}\left[f(x_{k},\xi_{k})-f(x_{*},\xi_{k})\right] & \leq-\mathbb{E}\left[V_{A_{k+1}}\left(-s_{k+1}\right)\right]+\frac{1}{2}\mathbb{E}\left[\sum_{t=0}^{k}\lambda_{t}^{2}\left\Vert \nabla f\left(x_{t},\xi_{t}\right)\right\Vert _{A_{t}^{-1}}^{2}\right]\\
 & +\mathbb{E}\left\langle \sum_{i=0}^{k}\lambda_{i}\nabla f\left(x_{i},\xi_{i}\right),x_{0}-x_{*}\right\rangle. 
\end{align*}
Now note that $s_{k+1}=\sum_{i=0}^{k}\lambda_{i}\nabla f\left(x_{i},\xi_{i}\right)$,
so: 
\begin{align*}
\mathbb{E}\left[V_{A_{k+1}}\left(-s_{k+1}\right)\right] & =\mathbb{E}\left[\max_{x}\left\{ \left\langle -s_{k+1},x-x_{0}\right\rangle -\frac{1}{2}\left\Vert x-x_{0}\right\Vert _{A_{k+1}}^{2}\right\} \right],\\
 & \geq\mathbb{E}\left[\left\langle -s_{k+1},x_{*}-x_{0}\right\rangle -\frac{1}{2}\left\Vert x_{*}-x_{0}\right\Vert _{A_{k+1}}^{2}\right],\\
 & =\mathbb{E}\left\langle \sum_{i=0}^{k}\lambda_{i}\nabla f\left(x_{i},\xi_{i}\right),x_{0}-x_{*}\right\rangle -\frac{1}{2}\left\Vert x_{*}-x_{0}\right\Vert _{A_{k+1}}^{2}.
\end{align*}
So combining this bound and further using the definition of $c_{k}$
and $\lambda_{k}$:
\begin{align*}
\frac{k+3/2}{3/2}\gamma\left(k+1\right)^{1/2}\mathbb{E}\left[f(x_{k})-f(x_{*})\right] & \leq\frac{1}{2}\mathbb{E}\left[\sum_{t=0}^{k}\lambda_{t}^{2}\left\Vert \nabla f\left(x_{t},\xi_{t}\right)\right\Vert _{A_{t}^{-1}}^{2}\right]+\frac{1}{2}\left\Vert x_{*}-x_{0}\right\Vert _{A_{k+1}}^{2}.
\end{align*}
To simplify further we need to start working in a coordinate wise
fashion. Let $D$ be the number of dimensions in $x$, then we can
write the above bound using Lemma \ref{lem:error_sum_bound} applied
coordinate wise as:
\begin{align*}
\frac{k+3/2}{3/2}\gamma\left(k+1\right)^{1/2}\mathbb{E}\left[f(x_{k})-f(x_{*})\right] & \leq\frac{1}{2}\sum_{d=0}^{D}\left(\mathbb{E}\left[\frac{3}{2}\lambda_{k}\left(\sum_{i=0}^{k}\lambda_{i}g_{id}^{2}\right)^{2/3}\right]\right)\\
 & +\frac{1}{2}\sum_{d=0}^{D}\left(x_{0x}-x_{*d}\right)^{2}\mathbb{E}\left(\lambda_{k+1}G^{2}+\sum_{i=0}^{k}\lambda_{i}g_{id}^{2}\right)^{1/3}.
\end{align*}
We now apply the bound $g_{id}\leq G$:
\begin{align*}
\frac{k+3/2}{3/2}\gamma\left(k+1\right)^{1/2}\mathbb{E}\left[f(x_{k})-f(x_{*})\right] & \leq\frac{3}{4}\sum_{d=0}^{D}\left(\lambda_{k}\left(\sum_{i=0}^{k}\lambda_{i}G^{2}\right)^{2/3}\right)\\
 & +\frac{1}{2}\sum_{d=0}^{D}\left(x_{0x}-x_{*d}\right)^{2}\left(\sum_{i=0}^{k+1}\lambda_{i}G^{2}\right)^{1/3}.
\end{align*}
Since $\lambda_{k}=\gamma\left(k+1\right)^{1/2}$, we can further
simplify using the summation property:
\[
\sum_{i=0}^{k}\left(i+1\right)^{1/2}\leq\frac{2}{3}\left(k+2\right)^{3/2},
\]

we apply on the two locations on the right to give: 
\begin{align*}
\frac{k+3/2}{3/2}\gamma\left(k+1\right)^{1/2}\mathbb{E}\left[f(x_{k})-f(x_{*})\right] & \leq\frac{1}{2}\gamma^{5/3}\sum_{d=0}^{D}\left(k+1\right)^{1/2}\left(k+2\right)G^{4/3}\\
 & +\frac{1}{3}\gamma^{1/3}\sum_{d=0}^{D}\left(x_{0x}-x_{*d}\right)^{2}\left(k+3\right)^{1/2}G^{2/3}.
\end{align*}
Note that: 
\begin{align*}
\frac{\left(k+3\right)^{1/2}}{(k+3/2)(k+1)} & \leq\frac{\left(k+3/2\right)^{1/2}+\left(3/2\right)^{1/2}}{(k+3/2)(k+1)}\\
 & \leq\frac{1}{k+1}+\frac{1}{(k+1)}\,\\
 & \leq\frac{2}{k+1}\,
\end{align*}
and likewise:
\[
\frac{k+2}{k+3/2}\leq2
\]
so after rearranging: 
\begin{align*}
\frac{2}{3}\mathbb{E}\left[f(x_{k})-f(x_{*})\right] & \leq2\gamma^{2/3}G^{4/3}D\\
 & +\gamma^{-2/3}G^{2/3}\frac{2}{k+1}\sum_{d=0}^{D}\left(x_{0x}-x_{*d}\right)^{2},
\end{align*}
\[
\mathbb{E}\left[f(x_{k})-f(x_{*})\right]\leq3\gamma^{2/3}G^{4/3}D+\frac{3}{k+1}\gamma^{-2/3}G^{2/3}\left\Vert x_{0}-x_{*}\right\Vert ^{2}.
\]
Taking the gradient with respect to $\gamma$ to zero gives 
\[
0=\frac{2}{3}\gamma^{-1/3}G^{4/3}D-\frac{2}{3(k+1)}\gamma^{-5/3}G^{2/3}\left\Vert x_{0}-x_{*}\right\Vert ^{2},
\]
\[
\therefore\gamma^{-1}G^{4}D^{3}=\frac{1}{\left(k+1\right)^{3}}\gamma^{-5}G^{2}\left\Vert x_{0}-x_{*}\right\Vert ^{6},
\]
\[
\therefore\gamma^{4}=\frac{1}{\left(k+1\right)^{3}D^{3}G^{2}}\left\Vert x_{0}-x_{*}\right\Vert ^{6},
\]
\[
\therefore\gamma=\frac{1}{\left(k+1\right)^{3/4}D^{3/4}G^{1/2}}\left\Vert x_{0}-x_{*}\right\Vert ^{3/2}.
\]
Using this optimal $\gamma$ gives: 
\[
\gamma^{2/3}=\frac{1}{k^{1/2}D^{1/2}G^{1/3}}\left\Vert x_{0}-x_{*}\right\Vert .
\]
and so: 
\begin{align*}
\mathbb{E}\left[f(x_{k})-f(x_{*})\right] & \leq\frac{6}{k^{1/2}}\left\Vert x_{0}-x_{*}\right\Vert GD^{1/2}.
\end{align*}
Note that $\left\Vert g\right\Vert _{2}\leq D^{1/2}\left\Vert g\right\Vert _{\infty}=D^{1/2}G$,
so the dependence on dimensionality here is comparable to standard
stochastic method proofs which have $\left\Vert g\right\Vert _{2}$
on the right instead.

\subsection{Time varying case}


Consider the situation where the bound on the gradient potentially
varies over time. 
\[
\left\Vert \nabla f(x_{i},\xi)\right\Vert _{\infty}\leq G_{i}\;\text{for all }x,\xi.
\]
Then using the same argument as in the previous section we arrive
at: 
\begin{align*}
\mathbb{E}\left[f(x_{k})-f(x_{*})\right] & \leq3\gamma^{2/3}\frac{1}{\left(k+1\right)}D\left(\sum_{i=0}^{k+1}\left(i+1\right)^{1/2}G_{i}^{2}\right)^{2/3}\\
 & +3\gamma^{-2/3}\frac{1}{\left(k+1\right)^{3/2}}\left\Vert x_{0}-x_{*}\right\Vert _{2}^{2}\left(\sum_{i=0}^{k+1}\left(i+1\right)^{1/2}G_{i}^{2}\right)^{1/3}.
\end{align*}
We may solve for the optimal step size, giving: 
\[
\gamma^{4/3}=\frac{1}{\left(k+1\right)^{1/2}}\frac{\left\Vert x_{0}-x_{*}\right\Vert _{2}^{2}\left(\sum_{i=0}^{k+1}\left(i+1\right)^{1/2}G_{i}^{2}\right)^{1/3}}{D\left(\sum_{i=0}^{k+1}\left(i+1\right)^{1/2}G_{i}^{2}\right)^{2/3}},
\]
\[
\therefore\gamma^{4/3}=\frac{1}{\left(k+1\right)^{1/2}}\frac{\left\Vert x_{0}-x_{*}\right\Vert _{2}^{2}}{D\left(\sum_{i=0}^{k+1}\left(i+1\right)^{1/2}G_{i}^{2}\right)^{1/3}},
\]
\[
\therefore\gamma^{2/3}=\frac{1}{\left(k+1\right)^{1/4}}\frac{\left\Vert x_{0}-x_{*}\right\Vert _{2}}{D^{1/2}\left(\sum_{i=0}^{k+1}\left(i+1\right)^{1/2}G_{i}^{2}\right)^{1/6}}.
\]
Then substituting this in gives: 
\begin{align*}
\mathbb{E}\left[f(x_{k})-f(x_{*})\right] & \leq6\frac{1}{\left(k+1\right)^{5/4}}D^{1/2}\left\Vert x_{0}-x_{*}\right\Vert _{2}\left(\sum_{i=0}^{k+1}\left(i+1\right)^{1/2}G_{i}^{2}\right)^{1/2}.
\end{align*}
When applying $\lambda_{i}=\gamma$, as in AdaGrad, we instead get:
\begin{align*}
\mathbb{E}\left[f(x_{k})-f(x_{*})\right] & \leq3\frac{\gamma^{1/2}}{\left(k+1\right)}D\left(\sum_{i=0}^{k+1}G_{i}^{2}\right)^{1/2}\\
 & +3\frac{1}{\left(k+1\right)\gamma^{1/2}}\left\Vert x_{0}-x_{*}\right\Vert _{2}^{2}\left(\sum_{i=0}^{k+1}G_{i}^{2}\right)^{1/2},
\end{align*}
solving for the optimal step size: 
\[
\frac{\gamma^{1/2}}{\left(k+1\right)}D\left(\sum_{i=0}^{k}G_{i}^{2}\right)^{1/2}=\frac{1}{\left(k+1\right)\gamma^{-3/2}}\left\Vert x_{0}-x_{*}\right\Vert _{2}^{2}\left(\sum_{i=0}^{k}G_{i}^{2}\right)^{1/2},
\]
\[
\therefore\gamma^{2}=\frac{\left\Vert x_{0}-x_{*}\right\Vert _{2}^{2}}{D}.
\]
So: 
\begin{align*}
\mathbb{E}\left[f(x_{k})-f(x_{*})\right] & \leq\frac{6}{\left(k+1\right)}\left\Vert x_{0}-x_{*}\right\Vert _{2}D^{1/2}\left(\sum_{i=0}^{k}G_{i}^{2}\right)^{1/2}.
\end{align*}

\section{Cube root formulation}

Consider the minimization problem parameterized by $g:k\times D$
and a single vector $s:D$, where 
\[
\min_{s}\sum_{i=0}^{k}\sum_{d=0}^{D}\frac{g_{id}^{2}}{s_{d}},\;\left\Vert s\right\Vert _{2}^{2}\leq c,\;\forall d:\,s_{d}>0
\]
In this section we show that $s_{d}\propto\sqrt[3]{\sum_{i=0}^{k}g_{id}^{2}}$
is a solution. Without loss of generality we disregard the inequality constraint
on $s_{d}$ and consider only positive solutions to the equality constrained
problem. We will apply the method of Lagrange multipliers.

Firstly we form the Lagrangian with multiplier $\mu$:
\[
L(s,\mu)=\sum_{i=0}^{k}\sum_{d=0}^{D}\frac{g_{id}^{2}}{s_{d}}+\frac{\mu}{2}\left(\sum_{d=0}^{D}s_{d}^{2}-c\right)
\]
Saddle-points of the Lagrangian can be found by equating the gradients
to zero and solving.
\[
\frac{\partial L}{\partial s_{d}}L(s,\mu)=-\frac{1}{s_{d}^{2}}\sum_{i=0}^{k}g_{id}^{2}+\mu s_{d},
\]
\[
\frac{\partial L}{\partial\mu}L(s,\mu)=\frac{1}{2}\left(\sum_{d=0}^{D}s_{d}^{2}-c\right).
\]
From the first equation:
\[
\frac{1}{s_{d}^{2}}\sum_{i=0}^{k}g_{id}^{2}=\mu s_{d},
\]
\[
\therefore s_{d}^{3}=\frac{1}{\mu}\sum_{i=0}^{k}g_{id}^{2}.
\]
Therefore $s_{d}=\mu^{-1/3}\left(\sum_{i=0}^{k}g_{id}^{2}\right)^{1/3}$
since $s_{d}$ is positive we take the positive root, The Lagrange
multiplier $\mu$ is given by the requirement that
\[
\sum_{d=0}^{D}s_{d}^{2}=c,
\]
\[
\therefore\mu^{2/3}=c^{-1}\sum_{d=0}^{D}\left(\sum_{i=0}^{k}g_{id}^{2}\right)^{2/3},
\]
\[
\therefore\mu^{1/3}=\sqrt{c^{-1}\sum_{d=0}^{D}\left(\sum_{i=0}^{k}g_{id}^{2}\right)^{2/3}}.
\]

So 
\[
s_{d}=\frac{1}{\sqrt{c^{-1}\sum_{d=0}^{D}\left(\sum_{i=0}^{k}g_{id}^{2}\right)^{2/3}}}\left(\sum_{i=0}^{k}g_{id}^{2}\right)^{1/3}.
\]

We can verify that this is an extreme point of the original problem
by noting that the linear independence constraint qualification (LICQ)
condition trivially holds when using one equality constraint. Since
the objective is convex for $s_{d}>0$, this point must be a minimizer.

\end{document}